\documentclass[10pt,journal,compsoc]{IEEEtran}

\usepackage{booktabs} 
\usepackage{times}  
\usepackage{helvet}  
\usepackage{courier}  
\usepackage{url}  
\usepackage{graphicx}  
\usepackage{footnote}
\usepackage{subfigure}
\usepackage{amsmath}
\usepackage{amsfonts}
\usepackage{multirow}
\usepackage{bm}
\usepackage{enumitem}
\usepackage{color}
\usepackage{nth}

\usepackage{amsthm}
\theoremstyle{plain}

\newtheorem*{theorem*}{Theorem}

\newcommand{\ie}{\emph{i.e., }}
\newcommand{\eg}{\emph{e.g., }}

\newcommand{\wrt}{\emph{w.r.t. }}

\newcommand{\aka}{\emph{aka. }}

\newcommand{\cten}[1]{\bm{\mathcal{#1}}}

\newlength\savedwidth

\def\BibTeX{{\rm B\kern-.05em{\sc i\kern-.025em b}\kern-.08em
		T\kern-.1667em\lower.7ex\hbox{E}\kern-.125emX}}
\begin{document}
\title{Cross-GCN: Enhancing Graph Convolutional Network with $k$-Order Feature Interactions}



%
\author{Fuli~Feng,
	Xiangnan~He,
	Hanwang~Zhang,
	Tat-Seng~Chua
	\IEEEcompsocitemizethanks{
		\IEEEcompsocthanksitem F. Feng and TS. Chua are with School of Computing, National University of Singapore, Computing 1, Computing Drive, 117417, Singapore.
		E-mail: fulifeng93@gmail.com, dcscts@nus.edu.sg.
		\protect\\
		\IEEEcompsocthanksitem X. He (corresponding author) is with School of Information Science and Technology, University of Science and Technology of China, Hefei, China.
		E-mail: xiangnanhe@gmail.com.
		\protect\\ 
		\IEEEcompsocthanksitem H. Zhang is with the School of Computer Science \& Engineering, Nanyang Technological University, Singapore. E-mail: see http://www.ntu.edu.sg/home/hanwangzhang/
	}
}

\markboth{IEEE TRANSACTIONS ON KNOWLEDGE AND DATA ENGINEERING, SUBMISSION 2019}{Feng \MakeLowercase{\textit{et al.}}: Cross-GCN: Enhancing Graph Convolutional Network with $k$-Order Feature Interactions}
\IEEEtitleabstractindextext{
\begin{abstract}
Graph Convolutional Network (GCN) is an emerging technique that performs learning and reasoning on graph data. It operates feature learning on the graph structure, through aggregating the features of the neighbor nodes to obtain the embedding of each target node. Owing to the strong representation power, recent research shows that GCN achieves state-of-the-art performance on several tasks such as recommendation and linked document classification.

Despite its effectiveness, we argue that existing designs of GCN forgo modeling cross features, making GCN less effective for tasks or data where cross features are important. Although neural network can approximate any continuous function, including the multiplication operator for modeling feature crosses, it can be rather inefficient to do so (\ie wasting many parameters at the risk of overfitting) if there is no explicit design. 

To this end, we design a new operator named \textit{Cross-feature Graph Convolution}, which explicitly models the arbitrary-order cross features with complexity linear to feature dimension and order size. We term our proposed architecture as \textit{Cross-GCN}, and conduct experiments on three graphs to validate its effectiveness. Extensive analysis validates the utility of explicitly modeling cross features in GCN, especially for feature learning at lower layers.
\end{abstract}

}

\maketitle
\IEEEdisplaynontitleabstractindextext
\IEEEpeerreviewmaketitle
\section{Introduction}
As a data structure, graph has been intensively used in information retrieval applications, ranging from search engines~\cite{xiong2018towards}, digtal libaries~\cite{wu2015citeseerx}, to recommender systems~\cite{ying2018graph,he2017birank} and question-answering systems~\cite{sun2015open}. For example, in Web search, the structure of Web-page graph has been mined to estimate the page importance~\cite{page1999pagerank}; in recommender systems, the interaction graph between users and items provides rich signal about collaborative filtering~\cite{he2017birank}. Along with the trend of deep learning, research on graph mining has been shifted from structure understanding to representation learning (\aka feature learning)~\cite{gao2018bine}, which offers a universal way to perform predictive analytics (\eg node classification, and link prediction) on discrete and high-dimensional graph data. 

\begin{figure}[t]
	\centering
	\includegraphics[width=0.48\textwidth]{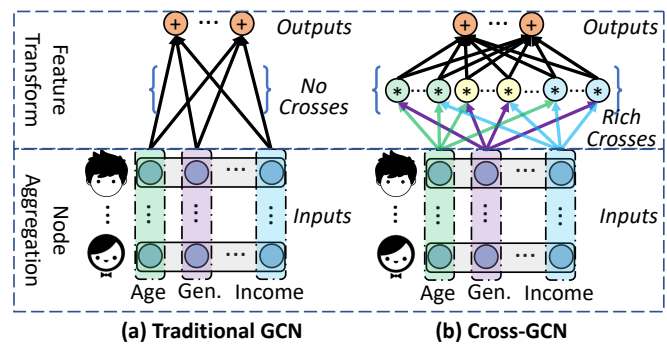}
	\caption{Illustration of traditional GCN and Cross-GCN. Traditional GCN performs feature transform with simple matrix mapping, which assumes no crosses between features. Our Cross-GCN augments feature transformation by modeling the rich crosses between features. Better viewed in color.}
	\label{fig:gcn}
\end{figure}
Owing to the extraordinary representation ability, \textit{Graph Convolutional Networks} (GCNs) have become a promising solution for representation learning over graphs~\cite{kipf2017semi,hamilton2017inductive,xu2019powerful}. Generally, GCNs learn node representations in a low-dimensional latent space from raw input features and node connections with multiple graph convolutional layers. As shown in Figure~\ref{fig:gcn}(a), a traditional GCN layer typically contains two components: \textit{node aggregation module} and \textit{feature transformation module}. The node aggregation module augments the representation of a targeted node via fusing information from its connected nodes based on the assumption that the properties of connected nodes would also reflect the property of the target node. The feature transformation module transforms the input features  into higher-level hidden features that better describe the node. Current research on GCNs mainly focuses on developing node aggregation modules emphasizing different connection properties such as local similarity~\cite{kipf2017semi}, multi-hop connectivity~\cite{defferrard2016convolutional}, and structural similarity~\cite{donnat2018learning}, nevertheless, while simply performs feature transformation with matrix mapping.

\begin{figure}[]
	\centering
	\mbox{
		\hspace{-0.15in}
		\subfigure[Training]{
			\label{fig:hp_dropout}
			\includegraphics[width=0.26\textwidth]{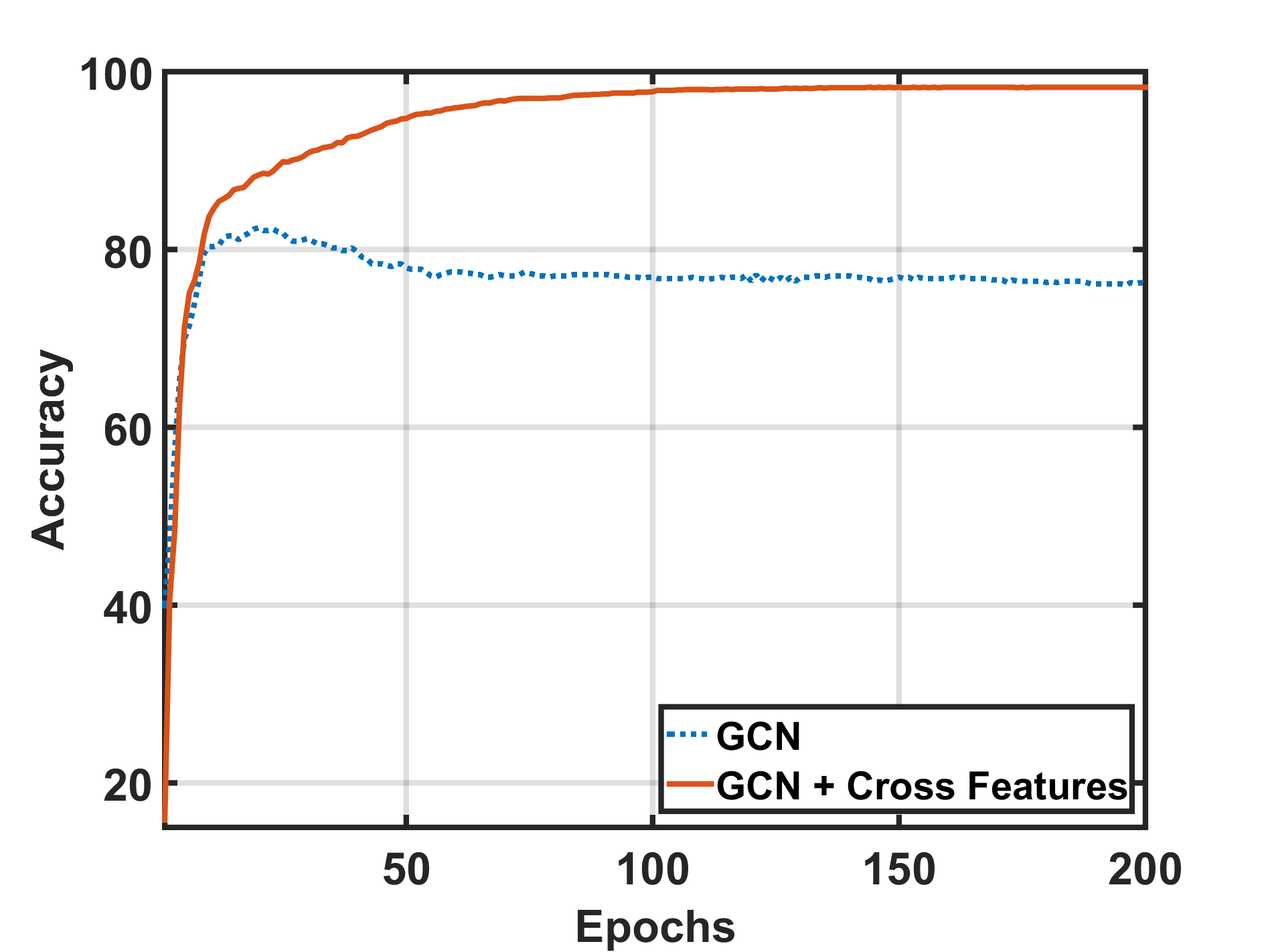}
		}
		\hspace{-0.25in}
		\subfigure[Testing]{
			\label{fig:hp_weight}
			\includegraphics[width=0.26\textwidth]{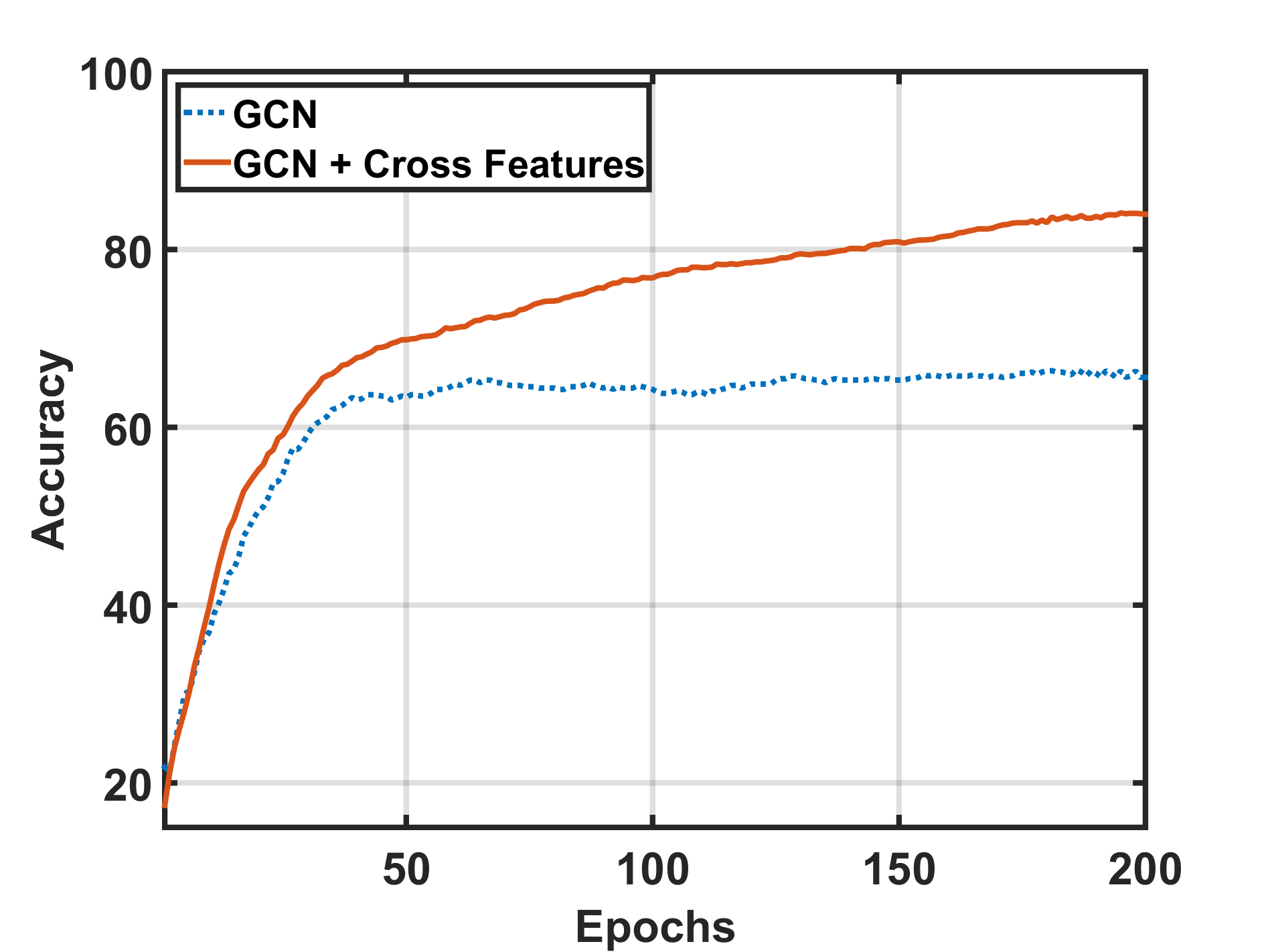}
		}
	}
	\caption{Training (left) and testing (right) accuracy on a citation graph (described in Section~\ref{sss:dataset}) of GCN (two layers) with original features and cross features as inputs.
	}
	\label{fig:gcn_vs_cgcn}
\end{figure}
We argue that existing GCN layers are hardly able to capture \textit{cross features}\footnote{In the following, we interchangeably use \textit{feature interaction} and \textit{cross feature}.}~\cite{lian2018xdeepfm}, which are essential for the success of many graph applications~\cite{pang2017deeprank,he2017neural,wang2017gated,wang2017deep}. For example, by capturing a cross feature: \textit{gender=female \& age=[20,25] \& income=\$10,000/month}, a recommender system could obtain better user representation and make more accurate product recommendations (\eg whether to purchase the latest version of iPhone)~\cite{wang2018tem}. To date, the simple matrix mapping used in most existing GCN layers, while equipped with non-linear activation, can hardly capture such cross features. As such, despite the extraordinary ability of GCNs to incorporate connection properties, we argue that their representation ability could be substantially improved by sharpening the feature transformation module to also encode cross features. 

Figure~\ref{fig:gcn_vs_cgcn} provides an empirical evidence on the weak ability of GCN in capturing cross features. We train a two-layer GCN on a citation graph with raw features and cross features as inputs, respectively. As can be seen, for node classification, training GCN with cross features achieves much higher accuracy on both training and testing sets than training with original features. It validates that explicitly considering cross features can enhance the representation ability of GCNs. Nevertheless, heuristic methods that either manually construct and select cross features or brutally enumerate all interactions are unaffordable for many real-world graphs with a large number of features, which motivates us to develop new graph convolution operations with feature interactions considered in GCN.

The main challenge of modeling cross feature is efficiency since the number of cross features grows dramatically with feature dimension. To tackle this challenge, we first devise a new operator named \textit{Cross-feature Graph Convolution}. The new operator contributes a new feature transformation module encoding cross features at arbitrary orders with a complexity linear to the feature dimension and order size. 
Thereafter, we develop a new solution for graph-based representation learning, named \textit{Cross-GCN} (illustrated in Figure~\ref{fig:gcn}(b)), by stacking multiple Cross-feature Graph Convolutional layers. Experiments on three graphs of node classification demonstrate the remarkable effectiveness of Cross-GCN. Moreover, Cross-GCN achieves more significant improvements on data with sparse low-level features, indicating its proper application scenarios.

The main contributions of this paper are summarized as:
\begin{itemize}[leftmargin=*]
	\item We devise a new graph-oriented operator, \textit{Cross-feature Graph Convolution}, which explicitly captures arbitrary-order feature interactions with linearly increased parameters and computation complexity.
	\item By equipping the Cross-feature Graph Convolution operator, we develop a new graph-based learning solution, named \textit{Cross-GCN}, which has the same model, computation, and memory complexity as vanilla GCN.
    \item With node classification as the testing task, we conduct experiments on three graphs and demonstrate the effectiveness of the proposed Cross-GCN (implementation will be released upon acceptance).
\end{itemize}

In the remainder of this paper, Section~\ref{s:pre} introduces preliminary knowledge about GCN and cross feature. In Section~\ref{s:met}, we elaborate the methodology followed by discussion of experiments and related work in Section~\ref{s:exp} and Section~\ref{s:rel}, respectively. Lastly, we conclude the paper, and envision some future work in Section~\ref{s:con}.

\section{Preliminary}
\label{s:pre}
We first introduce some notations used in the following sections. We use bold capital letters (e.g. $\bm{X}$) and bold lowercase letters (e.g. $\bm{x}$) to denote matrices and vectors, respectively. In addition, bold capital script letters (\eg $\bm{\mathcal{X}}$) 
are used to represent tensors. 
Note that all vectors are in a column form if not otherwise specified, and $X_{ij}$ denotes the entry of matrix $\bm{X}$ at the row $i$ and column $j$. Lastly, element-wise product and tensor outer product are denoted as $\odot$ and $\otimes$, respectively. In Table~\ref{tab:terms}, we summarize some of the terms and notations.

\begin{table}[]
	\caption{Terms and notations used in the paper.}
	\vspace{-0.3cm}
	\label{tab:terms}
	\resizebox{0.48\textwidth}{!}{%
		\begin{tabular}{c|l}
			\hline
			Symbol  & Definition \\ \hline \hline
			$\bm{X} \in \mathbb{R}^{N \times D}$     &    features of $N$ nodes.  \\ 
			$\bm{Y} \in \mathbb{R}^{M \times L}$ & labels of $M$ nodes. \\
			$\bm{A} \in \mathbb{R}^{N \times N}$ & adjacency matrix of a graph. \\
			$\bm{D} \in \mathbb{R}^{N \times N}$ & degree matrix of a graph. \\
			$neighbor(\bm{x})\}$ & neighbors of a node. \\
			$\mathcal{F}$ & feature transformation operation. \\
			$\mathcal{N}$ & node aggregation operation. \\
			$f^k(\cdot)$ & $k$-order cross-feature transformation function. \\
			$x^k_{i_1 i_2 \cdots i_k}$ & a $k$-order cross feature. \\
			$\cten{X}^k$ & a $k$-order tensor of cross features. \\
			$\bm{h}^k \in \mathbb{R}^{E}$ & hidden vector encodes $k$-order cross features. \\
			$\bm{W}$, $\bm{b}$ & parameters for 1-order feature transformation. \\
			$\mathcal{\mathbf{\Theta}}$ & all model parameters.\\
			$\odot$ & element-wise product. \\
			$\otimes$ & tensor outer product. \\
			$\sigma$ & a non-linear activation function. \\
			\hline
		\end{tabular}%
	}
\end{table}

\subsection{Graph Convolutional Networks}
A graph with $N$ nodes is typically represented as an adjacency matrix $\bm{A} \in \mathbb{R}^{N \times N}$ associated with a feature matrix $\bm{X}=[\bm{x}_1,\bm{x}_2,\cdots,\bm{x}_N]^{T}\in\mathbb{R}^{N \times D}$. $\bm{A}$ is a binary matrix, where $A_{ij} = 1$ if there is an edge between node $i$ and $j$, otherwise $A_{ij} = 0$. $D$ is the dimension of the input features which varies in a wide range from hundreds to millions in different applications. Taking the graph data as inputs, each GCN layer learns node representations in a low-dimensional embedding space (the last layer makes predictions)~\cite{kipf2017semi,hamilton2017inductive,xu2019powerful}. Note that a node feature in a latent layer is one dimension of previous layer's output. In the first layer, a node feature is one of the raw node features (\ie one dimension of $\bm{x}$). To simplify the notations and presentation, in the following, we always take the first layer as example to elaborate the detail of a GCN layer.

As shown in Figure~\ref{fig:gcn}(a), a GCN layer mainly performs two operations: \textit{node aggregation} and \textit{feature transformation}, to learn node representations. For a target node, a GCN layer can be abstracted as:
\begin{align}
	\bm{x} = \mathcal{N}(\{\bm{x}_{n} | n \in neighbor(\bm{x})\}),~~~ \bm{h} = \mathcal{F}(\bm{x}),
	\label{eq:gcn_layer}
\end{align}
where $\mathcal{F}$ and $\mathcal{N}$ denote the feature transformation and node aggregation operations, respectively. $\bm{x} \in \mathbb{R}^D$ is the feature representation of the target node after aggregating neighbor node features. $\bm{h} \in \mathbb{R}^E$ is the latent representation of the target node.

\textbf{Node aggregation} enriches the representation of a target node by aggregating information from its neighbor nodes. The rationale is that the properties of neighbor nodes could reflect the properties of a target node. For instance, in a social network, having connections with many nodes with male gender and interests in playing video games indicates that the target node might be a teenager boy. Note that $neighbor(\bm{x})$ could also includes the target node since self-connections~\cite{kipf2017semi} are typically intentionally added. 
Research on GCN has been focusing on developing different formats of $\mathcal{N}$ to distill distinct information from $neighbor(\bm{x})$~\cite{kipf2017semi,hamilton2017inductive,xu2019powerful}. For instance, an average pooling filter out common properties from connected nodes, while a max pooling would identify representative elements among connected nodes~\cite{xu2019powerful}. 

\textbf{Feature transformation} 
projects the target node from the input feature space into a high-level latent space in order to represent the node more comprehensively. In most existing GCNs, $\mathcal{F}$ is implemented as a matrix mapping equipped with a non-linear activation, which is formulated as:
\begin{align}
    \bm{h} = \mathcal{F}(\bm{x}) = \sigma(\bm{W} \bm{x} + \bm{b}),\text{ with } h_e = \sigma\left(\bm{w}_e^T \bm{x} + b_e\right),
    \label{eq:iso_fea_trans}
\end{align}
where $\sigma$ is an activation function; $\bm{W} \in \mathbb{R}^{E \times D}$ and $\bm{b} \in \mathbb{R}^{E}$ are mapping matrix and bias vector, respectively. Technically, the feature transformation operation composes the input features into the number of $E$ hidden features via biased weighted sum. The $e$-th row in $\bm{W}$ corresponds to the weights for calculating the $e$-th hidden feature $h_e$ with the corresponding bias $b_e$. We term feature transformation in Equation~\ref{eq:iso_fea_trans} as \textit{1-order feature transformation} since none cross feature is considered.

\subsection{Cross Feature}
\label{ss:cross_feature}
A $k$-order \textit{cross feature} (\aka feature interaction) is a combination of $k$ input features~\cite{lian2018xdeepfm}, which could be formulated as:
\begin{align}
	x^k_{i_1 i_2 \cdots i_k} = \prod_{i \in \{i_1 i_2 \cdots i_k\}} x_{i},
	\label{eq:cross_fea}
\end{align}
where $x_{i}$ is the $i$-th dimension of the input features. For instance, a $3$-order feature interaction could happen on three input features: \textit{gender}, \textit{age}, and \textit{income}. Such feature interaction could categorize users into more subtle groups and might benefit the modeling of user profiles such as interests~\cite{wang2018tem}. 
Given a feature vector $\bm{x} \in \mathbb{R}^D$ with $D$ features, we consider all $k$-order crosses among the $D$ features, which can be organized into a $k$-order tensor:
\begin{align}
	\cten{X}^k = \underbrace{\bm{x} \otimes \cdots \otimes \bm{x}}_{\text{$k$ copies}} \in \mathbb{R}^{D \times \cdots \times D}.
\end{align}

Similar as Equation~\ref{eq:iso_fea_trans}, a \textit{$k$-order cross-feature transformation} calculates $E$ hidden features from all the cross features in $\cten{X}^k$. For the $e$-th hidden feature, the formulation is:
\begin{align}
h^k_{e} = \sigma\left(sum(\cten{X}^k \odot \cten{W}^k_e) + b_e\right),
\label{eq:cro_fea_trans}
\end{align}
where $\cten{W}^k_e$ is a $k$-order tensor of parameters with the same size as $\cten{X}^k$ to perform the transformation. Note that $\cten{W}^k_e$ would become the $\bm{w}_e$ in Equation~\ref{eq:iso_fea_trans} when $k=1$. $b_e$ is the bias for the $e$-th hidden feature. $sum$ is a summation over all entries of a tensor. 
The computational complexity of the transformation is $O(D^k)$, which makes it is unaffordable when $D$ is large. Therefore, the key research challenge of cross-feature transformation is to devise an elegant mechanism that reduces the complexity while maintaining the representation capability. In the following, we omit the activation function $\sigma$ and the bias $\bm{b}$ to simplify the presentation.

\section{Cross-GCN}
\label{s:met}
In this section, we elaborate the detail of the proposed \textit{Cross-feature Graph Convolution} operator in Section~\ref{ss:cgc} followed by its complexity analysis in Section~\ref{ss:complexity}. Lastly, we discuss the advantages of the proposed operator over existing solutions for modeling cross features in GCN.

\subsection{Cross-feature Graph Convolution}
\label{ss:cgc}
The proposed operator \textit{Cross-feature Graph Convolution} follows the general formulation of a GCN layer (Equation~\ref{eq:gcn_layer}). That is to say, the operator also performs two operations: \textit{node aggregation} and \textit{feature transformation}, to learn hidden node representations.
\begin{itemize}[leftmargin=*]
    \item 
    To efficiently and appropriately encode cross features, we devise a new \textit{cross-feature transformation module}, which could calculate cross-feature transformation at arbitrary orders with complexity linear to feature dimension and order size. 
    \item As our main target is to model cross features, we employ the node aggregation modules in existing GCNs\footnote{It is worthwhile to mention that we could also model feature interaction in node aggregation. That is to say, we consider interactions among different nodes. For simplicity, in this work, we only model feature interaction in feature transformation and will explore feature interaction in node aggregation in future work.}. For instance, we could use the \textit{average pooling}~\cite{hamilton2017inductive}, \textit{LSTM}~\cite{hamilton2017inductive}, and Attention Networks~\cite{velickovic2018graph}.
\end{itemize}


\subsubsection{Cross-feature Transformation Module}
Incorporating cross features has been found to achieve significant success in many applications such as search engines~\cite{pang2017deeprank}, QA systems~\cite{wang2017gated} and recommender systems~\cite{lian2018xdeepfm}. This is because the cross features could enrich the representation of a wide range of entities including Web articles and user-item pairs. Inspired by its success, Cross-feature Graph Convolution performs cross-feature transformation with order size up to $K$ to incorporate cross features. That is to say, the \textit{cross-feature transformation module} simultaneously performs cross-feature transformations at orders from 1 to $K$
, and fuses the outputs of the $K$ transformations into an overall output.

The key novelty of the proposed module is enabling cross-feature transformation in a linear complexity. Instead of directly evaluating Equation~\ref{eq:cro_fea_trans}, which has a polynomial complexity of $O(D^k)$, we propose a new operator by performing cross feature transformation in a recursive way:
\begin{align}
\label{eq:cro_fea_tran_mod}
\bm{h}^k = f^k(\bm{x}) = (\bm{W}^k \bm{x}) \odot f^{k - 1}(\bm{x}), \text{ and } f^0(\bm{x}) = \bm{1}.
\end{align}
$\bm{x} \in \mathbb{R}^{D}$ and $\bm{h}^k \in \mathbb{R}^{E}$ denotes the inputs and outputs of the $k$-order cross-feature transformation module, respectively. $\bm{W}^k \in \mathbb{R}^{E \times D}$ is a matrix of transformation parameters. 
$f^k(\cdot)$ denotes a \textit{$k$-order cross-feature transformation function} which encodes $k$-order cross features into a hidden vector $\bm{h}^k \in \mathbb{R}^{E}$. For instance, $f^2(\bm{x}) = (\bm{W}^2 \bm{x}) \odot (\bm{W}^1 \bm{x})$ performs the $2$-order cross-feature transformation.

In the next, we prove that the proposed Equation 6 provides a low-rank approximation to the $k$-order cross-feature transformation of Equation 5.
\begin{theorem*} 
Let $\mathcal{C}: \mathbb{R}^D \rightarrow \mathbb{R}^{E}$ be a $k$-order cross-feature transformation $\bm{h}^k = \mathcal{C}(\bm{x})$, which calculates each output dimension $h^k$ as: $h^k = sum(\cten{X}^k \odot \cten{W}^k)$ (Equation~\ref{eq:cro_fea_trans})\footnote{Note that we omit the subscript of dimension $e$ to simplify the notation.}. We have: $\mathcal{C}(\bm{x}) \approx (\bm{W}^k \bm{x}) \odot \cdots \odot (\bm{W}^1 \bm{x})$ where $\bm{W}^k \in \mathbb{R}^{E \times D}$.
\end{theorem*}

\begin{proof} To facilitate understanding, we first take the cross-feature transformation at \textit{second-order} as an example to present the derivation. Originally, $2$-order cross-feature transformation (Equation~\ref{eq:cro_fea_trans}) calculates a hidden feature as:
\begin{align}
    h^2 = sum(\cten{X}^2 \odot \cten{W}^2) = sum((\bm{x} \bm{x}^T) \odot \bm{W})
    \label{eq:2nd_cro_fea_trans}
\end{align}
where $\bm{W} \in \mathbb{R}^{D \times D}$ is the mapping matrix. Note that $h^2$ is just one dimension of $\bm{h}^2 \in \mathbb{R}^E$. In total, the number of needed parameters to calculate $\bm{h}^2$ would be $E \cdot D^2$. The key of accelerating the transformation is to reduce the number of parameters. 
According to the theory of matrix factorization~\cite{koren2009matrix}, we can decompose parameter matrix ($\bm{W}$) into latent factors $\bm{w}$ and $\bar{\bm{w}} \in \mathbb{R}^{D}$ subject to $\bm{W} \approx \bm{w} \bar{\bm{w}}^T$. Subsequently, we can approximate Equation~\ref{eq:2nd_cro_fea_trans} as:
\begin{align}
   h^2 \approx sum\left((\bm{x} \bm{x}^T) \odot (\bm{w} \bar{\bm{w}}^T)\right)  = sum\left((\bm{w} \odot \bm{x}) (\bar{\bm{w}} \odot \bm{x})^T\right).
    \label{eq:transform_1}
\end{align}
Furthermore, we can transform it into:
\begin{align}
    h^2 \approx & \sum_{i = 1}^D \sum_{j = 1}^D w_i x_i * \bar{w}_j x_j = \sum_{i = 1}^D w_i x_i * \sum_{j = 1}^D \bar{w}_j x_j \notag \\
    = & \sum_{i = 1}^D w_i x_i * (\bar{\bm{w}}^T \bm{x}) = (\bm{w}^T \bm{x}) (\bar{\bm{w}}^T \bm{x}).
    \label{eq:transform_2}
\end{align}
Following the same procedure, we can transform the calculation of all $E$ hidden features into the format based on latent factors (\ie $(\bm{w}^T \bm{x}) (\bar{\bm{w}}^T \bm{x})$). In total, we need $E$ pairs of latent factors. By organizing the latent factors into two parameter matrices $\bm{W}$ and $\bar{\bm{W}} \in \mathbb{R}^{E \times D}$ (one row for each hidden feature), the $2$-order cross-feature transformation can be transformed into:
\begin{align}
    \bm{h}^2 = (\bm{W} \bm{x}) \odot (\bar{\bm{W}} \bm{x}),
    \label{eq:sec_order_pr}
\end{align}
which has the same form as the proposed transformation module $\bm{h}^2  = f^2(\bm{x}) = (\bm{W}^2 \bm{x}) \odot (\bm{W}^1 \bm{x})$.

$k$-\textit{order} ($k \geq 2$). For a $k$-order cross-feature transformation, directly calculating a hidden feature needs the number of $D^k$ parameters ($\cten{W}^k$ in Equation~\ref{eq:cro_fea_trans}). 
By extending the matrix factorization to canonical tensor factorization~\cite{hitchcock1927expression}, we can factorize $\cten{W}^k$ into $k$ vectors of latent factors with dimension of $D$: $\{\bm{w}^{1}, \cdots, \bm{w}^{k}\}$ subject to $\cten{W}^k \approx \bm{w}^{1} \otimes \cdots \otimes \bm{w}^k$. Following Equation~\ref{eq:transform_1} and \ref{eq:transform_2}, we can calculate $h^k$ by:
\begin{align}
    h^k = & sum\left(\cten{X}^k \odot \cten{W}^k\right) = sum\left((\bm{w}^k \odot \bm{x}) \otimes \cdots (\bm{w}^1 \odot \bm{x})\right) \notag \\
    = & ({\bm{w}^k}^T \bm{x}) \odot \cdots \odot ({\bm{w}^1}^T\bm{x}),
\end{align}
where the number of parameters is $kD$. 
Similarly, we use $E$ sets of latent factors to calculate $E$ hidden features. We organize the latent factors into $k$ different matrices, and obtain: $\bm{h}^k = (\bm{W}^k \bm{x}) \odot \cdots \odot (\bm{W}^1 \bm{x})$ which is the proposed $f^k(\bm{x})$. 
\end{proof}
It is worthwhile to emphasize that the proposed module does not explicitly conduct the decomposition --- insteading of learning a costly $k$-order tensor $\cten{W}^k \in \mathbb{R}^{D \times \cdots \times D}$, we learn $k$ latent vectors $\{\bm{w}^{1}, \cdots, \bm{w}^{k}\}$, which can be seen as approximating $\cten{W}^k$ with canonical tensor factorization~\cite{hitchcock1927expression}.


\subsubsection{Order Aggregation}
Research on cross features has shown that cross features in different orders might have different impacts on the prediction~\cite{lian2018xdeepfm}. In other words, in the cross-feature transformation, hidden features transformed from cross features at different orders might contribute differently to the output. 
As such the cross-feature transformation module further performs order aggregation which is denoted as a function $a(\cdot)$. The order aggregation function $a(\cdot)$ first aggregates hidden features transformed from interactions at different orders ($\{\bm{h}^k | k = 1, \cdots, K\}$) into a single vector. We can perform the aggregation using many operations such as a pooling function (\eg mean pooling) or a neural modeling component (\eg LSTM~\cite{hochreiter1997long}). After aggregation, $a$ then adds bias and performs non-linear activation on the hidden vector to obtain the output ($\bm{h}$). 

In this work, to avoid introducing additional model parameters, we simply implement the order aggregation function $a(\cdot)$ as,
\begin{align}
\bm{h} = a(\{\bm{h}^k | k = 1, \cdots, K\}) = \sigma\left(\sum_{k=1}^{K} \alpha_k \bm{h}^k\right) + \bm{b}.
\end{align}
$\alpha_k$ regulates the information from $k$-order cross features flowing into the final node representation. When assigned with a larger value, the $k$-order cross features contribute more to the final output $\bm{h}$. We leave the exploration of advanced implementations of $a(\cdot)$ to future work since this work is focused on the cross-feature transformation.

\subsection{Complexity Analysis}
\label{ss:complexity}
For practical applications, in addition to effectiveness, the usability of a neural network operator also depends on three complexity criteria: 1) \textit{model complexity}, 2) \textit{computation complexity}, and \textit{memory complexity}. As such, we carefully analyze the complexity of the proposed cross-feature transformation module. Note that we omit the cost of order aggregation function $a(\cdot)$ since it is only affected by the output dimension ($E$) which is typically small. We focus on the complexity of the transformation functions ($f^k(\cdot)$) which is sensitive to input dimension ($D$) where $E \ll D$. 
\begin{itemize}[leftmargin=*]
    \item \textbf{Model complexity}. As shown in Equation~\ref{eq:cro_fea_trans}, a cross-feature transformation module considering $K$-order cross features contains $K$ parameter matrices $\{\bm{W}^k | k = 1, \cdots, K\}$ where $\bm{W}^k \in \mathbb{R}^{E \times D}$. Therefore, the overall number of model parameters is $O(KED)$ which increases linearly with feature dimension and order size. In addition, it means that the model parameters is $K$ times as many as that of the conventional feature transformation module (Equation~\ref{eq:iso_fea_trans}) in existing GCNs.
    \item \textbf{Computation complexity}. The computation cost mainly comes from the $K$ times matrix multiplication between $\bm{W}^k$ and $\bm{x}$. The computation complexity is $O(KED)$ which increases linearly with $D$. It should be noted that the computation can be easily accelerated by simply computing the $K$ matrix multiplications in parallel.
    \item \textbf{Memory complexity}. The main memory cost comes from storing model parameters which have the complexity of $O(KED)$. Note that most neural networks rely back-propagation algorithm (BP)~\cite{hecht1992theory} to learn the model parameters. As such, in the training phase, storing gradients would bring additional memory cost of which the complexity depends on the implementation of BP. Roughly, the additional cost should also be in order of $O(KED)$.
\end{itemize}
We conclude that the proposed cross-feature transformation module has complexities linear to feature dimension and order size. In addition, the overhead of the proposed module as compared to the conventional feature transformation module (Equation~\ref{eq:iso_fea_trans}) is linearly dependent on $K$. Considering that order size ($K$) is a hyper-parameter typically smaller than 5, the overhead should be acceptable.

\textbf{Discussion}. 
To the best of our knowledge, \textit{Graph Isomorphism Network} (GIN)~\cite{xu2019powerful} adopting a Multi-layer Perceptron (MLP) as the feature transformation module in each layer is the only existing GCN having the potential to capture cross features. The reason is that a MLP can approximate any arbitrary transformation according to the universal approximation theorem~\cite{hornik1991approximation}. However, existing research~\cite{andoni2014learning,beutel2018latent} has shown that it might take a large number of hidden units to appropriately approximate feature interactions, which could be much larger than the dimension of input features (\ie $D$). Therefore, the computational overhead is unaffordable. In addition, implicitly modeling feature interactions with MLP may lead to downsides that a GIN layer cannot control the maximum order of feature interactions and the strength of different orders. 
\section{Experiments}
\label{s:exp}
We evaluate the proposed method in the node classification task, which covers a wide range of applications such as user profiling~\cite{qiu2018deepinf}, fraud detection~\cite{li2014search}, and text classification~\cite{yao2019graph}. Note that it could be easily adapted to solve the other tasks such as link prediction and community detection. In the problem setting of node classification, a graph $G = \{\bm{A}, \bm{X}\}$ with $N$ nodes, associated with labels ($\bm{Y}$) of a portion of nodes. For simplification, we index the labeled nodes and unlabeled nodes in the range of $[1, M]$ and $(M, N]$. We train our model on the labeled nodes by optimizing: 
\begin{align}
	\Gamma = \sum_{i = 1}^M l(\bm{\hat{y}}_i, \bm{y}_i) + \lambda \|\bm{\Theta}\|_F^2,
\end{align}
where $l$ is a classification loss function such as cross-entropy. $\bm{\Theta}$ denotes all the model parameters.
\subsection{Experimental Settings}
\subsubsection{Datasets}
\label{sss:dataset}
In accordance with~\cite{kipf2017semi}, we test our model on citation graph.
\begin{itemize}[leftmargin=*]
	\item \textbf{Citation graphs.} Citation graphs represent documents and their citation relations with nodes and edges where features (extracted from document contents) and labels (topic of the document) are associated with each node. Following~\cite{kipf2017semi}, we adopt three citation graphs, \textbf{Cora}, \textbf{Citeseer} and \textbf{Pubmed}~\cite{sen2008collective}, with sparse bag-of-word document features. We follow the extreme data split in~\cite{kipf2017semi}, that is, 20 labeled nodes per class are used for training; 500 and 1,000 nodes are used for validation and testing, respectively.
	\item \textbf{Citeseer-Cross.} Moreover, we construct a semi-real dataset based on the \textbf{Citeseer} citation graph to test whether cross feature are properly considered. We intentionally compile features for nodes so that labels are highly correlated with $2$-order cross features. In particular, 
	we use a feature vector $\bm{x}$ with dimension of 12 to describe each node, which is initialized with random values sampled from a standard Gaussian distribution. 
	We then intentionally change the sign of feature values according to the label of a node so that there would be a cross feature that clearly indicates the value of label. For a node in class $c \in [1, 6]$, \ie $y_c = 1$, the edited feature would satisfy:
	$$ \left\{
	\begin{aligned}
	& x_{i * 2 - 1} * x_{i * 2} > 0, &~i = c~(y_i = 1), \\
	& x_{i * 2 - 1} * x_{i * 2} \leq 0, &~ i \neq c \& i \in [1, 6]~(y_i = 0). 
	\label{eqn:synthetic_rule}
	\end{aligned}
	\right.
	$$
	Note that the sign of a cross feature (\eg $x_1 * x_2$) is highly correlated with the label of a class (\eg $c = 1$). As such, a method would benefit from modeling 2-order feature interactions. Note that \textbf{Citeseer-Cross} and \textbf{Citeseer} have the same graph structure and data splits.
\end{itemize}
\begin{table}[]
	\caption{Statistics of the experimental datasets.}
	\label{tab:dataset}
	\resizebox{0.48\textwidth}{!}{%
		\begin{tabular}{c|ccccc}
			\hline
			Dataset  & \#Nodes & \#Edges & \#Classes & \#Features & Label rate \\ \hline \hline
			Citeseer & 3,312   & 4,732   & 6         & 3,703      & 0.036      \\ 
			Pubmed     &  19,717   & 44,338   & 3         & 500      & 0.003      \\ 
			Reddit400K     & 30,000  & 386,742 & 41       & 602      & 0.6      \\ 
			Citeseer-Cross & 3,312   & 4,732   & 6         & 12      & 0.036      \\ \hline
		\end{tabular}%
	}
\end{table}
In Table~\ref{tab:dataset}, we summarize the statistics of all experimental datasets.
\subsubsection{Methods}
\begin{itemize}[leftmargin=*]
	\item \textbf{SemiEmb}~\cite{weston2012deep}: It is a representative graph-based learning method based on Laplacian regularization which encourages connected nodes to have close embeddings and predictions. Here, we use a MLP with graph Laplacian regularization on the predictions.
	\item \textbf{DeepWalk}~\cite{perozzi2014deepwalk}: DeepWalk is a widely used graph representation learning method based on the skip-gram technique, which learns node representation by predicting node contexts that are generated by performing random walk on graph. The learned embeddings are fed into a MLP for label prediction.
	\item \textbf{GCN}~\cite{kipf2017semi}: is a general GCN model performing feature transformation with matrix mapping without explicit feature interactions, and node aggregation with a pooling function.
	\item \textbf{GIN}~\cite{xu2019powerful}: GIN is a generalization of vanilla GCN, performing feature transformation with a MLP in each convolution layer.
	\item \textbf{Cross-GCN}: is the simplest implementation of the proposed method with $K=2$ as order size.
\end{itemize}
Note that we omit potential graph convolution-based baselines revolving on the modeling of node aggregation such as~\cite{velickovic2018graph,atwood2016diffusion} since we focus on the modeling of feature transformation. In addition, for each model, we only test the single layer and two layer versions leaving the exploration of deeper models in future work.

\subsubsection{Parameter Settings}
We implement \textbf{Cross-GCN}, \textbf{GIN}\footnote{We remove the batch-normalization and decay of learning rate applied in~\cite{xu2019powerful} to speed up the convergence.} and \textbf{SemiEmb} with Tensorflow 1.8.0\footnote{\url{https://www.tensorflow.org/versions/r1.8/}.} and adopt the public implementation of \textbf{GCN}\footnote{\url{https://github.com/tkipf/gcn}.}. Following the implementation of GCN, we apply dropout and weight decay on each layer to prevent overfitting. 
To avoid over-tuning of hyper-parameters, we set \textbf{Cross-GCN} and \textbf{GIN} with the optimal hyper-parameters of \textbf{GCN} released in the original paper~\cite{kipf2017semi} except the size of hidden layers ($E$). In other words, we set dropout ratio ($\rho$) and weight decay ($\lambda$) for the $L_2$-norm with values of 0.5 and $5e\text{-}4$, respectively. For the size of hidden layer ($E$), we perform grid-search within the range of [16, 32, 64, 128].
All the models are optimized via Adam~\cite{kingma2014adam} with initial learning rate of 0.01. Following prior work~\cite{wu2019simplifying}, we also test the methods on 20 random splits of the training set while keeping the validation and test sets unchanged\footnote{Note that we repeat the test 20 times rather than 10 as the original paper~\cite{wu2019simplifying} to avoid large variance. The original paper avoids large variance by removing runs with outlier performance.}. For each run, the epoch a model achieves best performance on the validation is selected to report performance.

\subsection{Citeseer-Cross}
\subsubsection{Performance Comparison}
\begin{table}[]
	\caption{Performance comparison of GCN, GIN, and Cross-GCN with single layer and two layers \wrt test accuracy on the Citeseer-Cross dataset.}
	\label{tab:syngraph_all}
	\centering
	\resizebox{0.35\textwidth}{!}{%
		\begin{tabular}{c|cc}
			\hline
			& Single Layer & Two Layers \\ \hline \hline
			GCN       & 25.1$\pm$1.4     & 62.9$\pm$1.7   \\ 
			GIN       & 48.6$\pm$1.0     & 45.7$\pm$7.9   \\ 
			Cross-GCN & \textbf{76.9$\pm$0.8}     & \textbf{76.2$\pm$0.7}   \\ \hline
		\end{tabular}%
	}
\end{table}
Recall that the target of this work is to explicitly consider cross features in GCN. To investigate the effect of modeling feature crosses, we first test \textbf{GCN}, \textbf{GIN}, and \textbf{Cross-SGCN} on the semi-real dataset (Citeseer-Cross). Table~\ref{tab:syngraph_all} presents the performance of the models with single layer and two layers. Note that the main difference across the compared models is the feature transformation module. From the results, we have the following observations:
\begin{itemize}[leftmargin=*]
	\item In all cases, the proposed \textbf{Cross-GCN} outperforms \textbf{GCN} with significant improvements. This validates the effectiveness of the cross-feature transformation module. Moreover, this result demonstrates the benefit of explicitly modeling feature crosses in each layer as compared to relying totally on the multi-layer structure.
	\item \textbf{Cross-GCN} also beats \textbf{GIN} which indicates the advantages of the proposed cross-feature transformation module over a MLP in modeling feature crosses. Similar observation has been presented in previous research~\cite{beutel2018latent}. Such observations suggest that we should pay more attention on explicit cross-feature modeling.
	\item The performance of \textbf{Cross-GCN} with a single layer surpasses \textbf{GCN} \textbf{GIN} with two layers. This result further verifies the benefit of explicit cross-feature modeling. In addition, it indicates that modeling feature crosses would help to reduce network depth in some situations, which would facilitate model optimization.
\end{itemize}

\subsubsection{Impact of Model Complexity}
\begin{figure}[t]
	\centering
	\includegraphics[width=0.33\textwidth,trim={0 0 0 0cm}, clip]{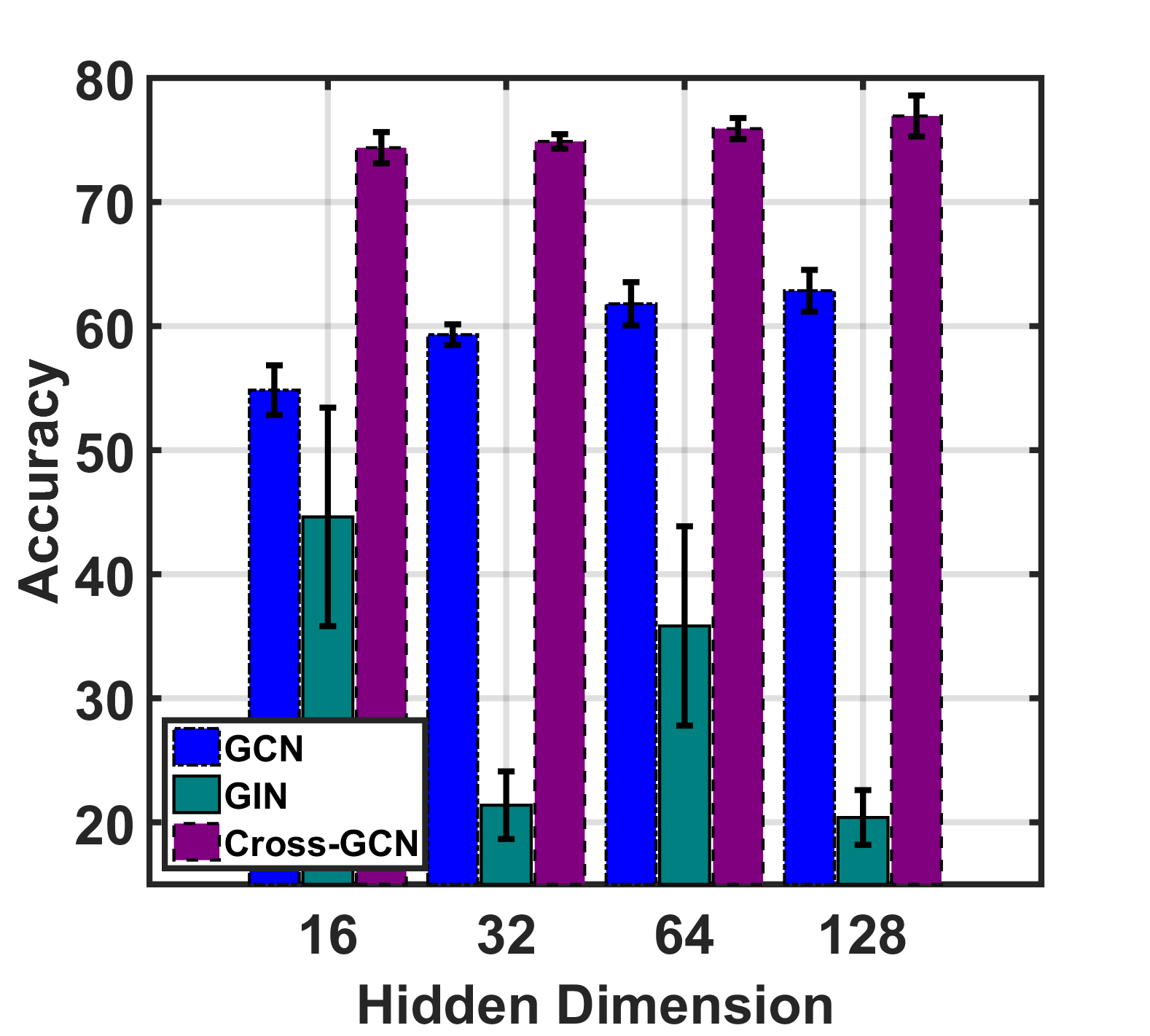}
	\caption{Performance comparison of GCN, GIN, and Cross-GCN under different size of hidden layer.
	}
	\label{fig:syngraph_hidden1}
\end{figure}
We then investigate the impact of model complexity by testing \textbf{GCN}, \textbf{GIN}, and \textbf{Cross-GCN} with different numbers of hidden units ($E$ for the first layer). Figure~\ref{fig:syngraph_hidden1} shows the performance \wrt test accuracy with $E = [16, 32, 64, 128]$. From the results, we have the following observations:
\begin{itemize}[leftmargin=*]
    \item In all cases, \textbf{Cross-GCN} achieves significantly better performance than \textbf{GCN} and \textbf{GIN}. It further justifies the effectiveness of the proposed cross-feature transformation module.
    \item Moreover, \textbf{Cross-GCN} with $E=16$ shows inspiring performance, which is consistently better than \textbf{GCN} and \textbf{GIN} with much larger hidden layers (\eg $E = 64$ and $E = 128$). This result further justifies the strong representation ability of \textbf{Cross-GCN} and indicates that the advanced representation ability is due to considering cross features rather than additional model parameters. Note that \textbf{Cross-GCN} doubles the model parameter of \textbf{GCN} when they have the same hidden layer size.
    \item \textbf{GIN} does not perform as well as expected when $E > 16$. We find that, during training, it always sticks at some sub-optimal points even though we try different learning rates and techniques like batch-normalization. It indicates that there are more optimization issues as we increase the depth of \textbf{GIN}. We leave the exploration of solutions in future work.
\end{itemize}

\subsubsection{Cross Feature at Different Layers}

\begin{table}
	\caption{Performance comparison of Cross-GCN on the Citeseer-Cross dataset with cross-feature considered at different layer.}
	\label{tab:syngraph_gate}
	\centering
	\resizebox{0.3\textwidth}{!}{
		\begin{tabular}{cc|c}
			\hline
			\multicolumn{2}{c|}{Feature Interaction} & \multirow{2}{*}{Accuracy}       \\ \cline{1-2} 
			$\nth{1}$ Layer & $\nth{2}$ Layer &        \\ \hline\hline
			0 & 0 & 62.9$\pm$1.7 \\
			0 & 1 & 63.3$\pm$1.0 \\
			1 & 0 & \textbf{78.9$\pm$2.0} \\ 
			1 & 1 & 76.2$\pm$1.0 \\ \hline
		\end{tabular}
	}
\end{table}

\begin{figure}[t]
	\centering
	\includegraphics[width=0.33\textwidth]{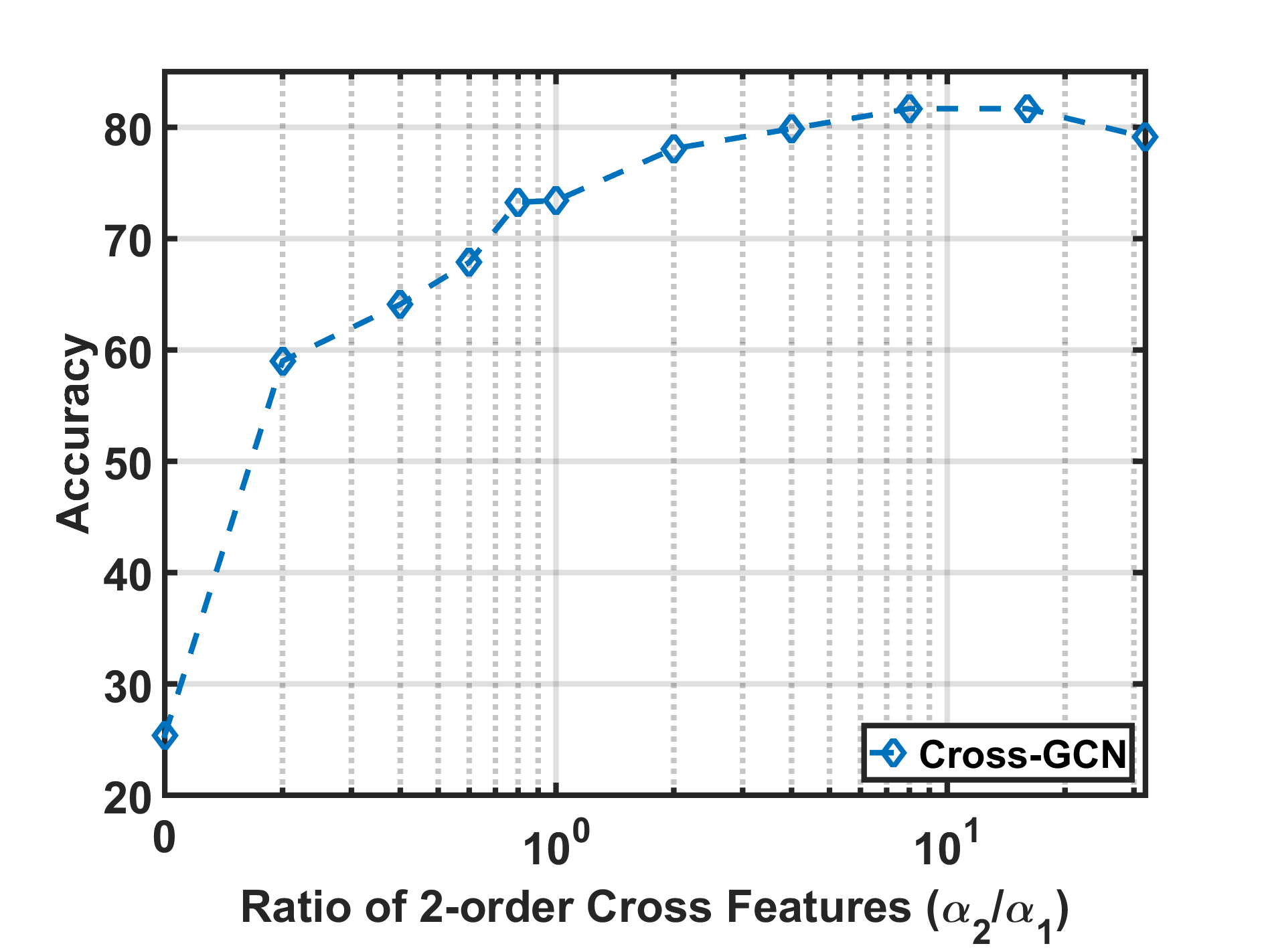}
	\vspace{-0.4cm}
	\caption{Impact on the performance of Cross-GCN with a single layer as adjusting the contribution of cross features.}
	\label{fig:syngraph_alpha}
\end{figure}


We then investigate the impact of considering cross features at different layers by performing ablation study. Table~\ref{tab:syngraph_gate} shows the performance of different combinations of \textbf{Cross-GCN} without feature crosses (row 1) as well as with feature crosses at the second layer only (row 2), first layer only (row 3), and both the first and second layers (row 4). From the table, we have the following observations:
\begin{itemize}[leftmargin=*]
    \item \textit{Cross-GCN} without consideration of cross features (row 1) achieves the worst performance. Once again, it demonstrates the necessity of performing cross feature modeling.
    \item There is a clear gap between the performance of considering cross features at the first layer (row 3) and the second layer (row 2). This result indicates that the utility of modeling feature crosses varies across layers and suggests modeling feature crosses at lower layers.
    \item \textit{Cross-GCN} achieves the best performance when considering cross features at the first layer only. The result is reasonable since the semi-real dataset is intentionally designed to benefit modeling of second-order interactions on raw features.
\end{itemize}

\subsubsection{Impact of Cross Features}
We investigate the impact of cross features by adjusting the values of $\alpha_1$ and $\alpha_2$ that control the contributions of 1-order feature transformation and 2-order cross-feature transformation to the output representation. Figure~\ref{fig:syngraph_alpha} shows the performance of \textit{Cross-GCN} (single layer) as we set $\alpha_1 = 1$ and adjust the value of $\alpha_2$ from 0 to 32. As can be seen, the prediction accuracy increases substantially (roughly from 20\% to 60\%) when $\alpha_2/\alpha_1$ is in the range of [0, 1], while is relatively stable when $\alpha_2/\alpha_1 > 1$. Findings on this results are twofold: 1) it demonstrates the necessity of regulating the contribution of feature interactions; and 2) we can simply set both $\alpha_1$ and $\alpha_2$ with fixed values of 1.0 to get an acceptable prediction performance. 

\subsection{Real-world Datasets}
\begin{table}[]
	\caption{Performance of the compared methods on the three real-world datasets \wrt testing accuracy.}
	\label{tab:real_perf}
	\centering
	\resizebox{0.45\textwidth}{!}{%
	\begin{tabular}{c|ccc}
		\hline
		Method & Citeseer & Cora & Pubmed \\ \hline \hline
		SemiEmb & 59.6 & 59.0 & 71.1 \\ 
		DeepWalk & 43.2 &67.2 &  65.3 \\ \hline
		GCN1 & 68.9$\pm$2.2 & 70.3$\pm$3.4 & 72.1$\pm$2.8 \\ 
		GIN1 & 62.4$\pm$2.9 & 74.3$\pm$2.4 & 74.5$\pm$2.3 \\ 
		Cross-GCN1\_Fix & 70.1$\pm$1.8 & 72.7$\pm$2.4 & 74.2$\pm$2.6 \\
		Cross-GCN1 & 68.9$\pm$1.9 & 72.9$\pm$2.2 & 74.8$\pm$2.7 \\ \hline 
		GCN & 69.7$\pm$2.0 & 79.1$\pm$1.8 & 77.6$\pm$2.0 \\ 
		GIN & 68.9$\pm$2.0 & 78.5$\pm$1.9 & 78.7$\pm$1.6 \\ 
		Cross-GCN\_Fix & \textbf{71.3$\pm$1.7} & 78.6$\pm$1.8 & \textbf{79.3$\pm$1.8} \\
		Cross-GCN & 69.6$\pm$2.2 & \textbf{79.8$\pm$1.6} & 78.8$\pm$1.8 \\ \hline 
	\end{tabular}%
	}
\end{table}
We further test the GCN models on three real-world datasets where the prediction performance is not only determined by the consideration of feature crosses.
\subsubsection{Performance Comparison}
Here, we compare three more methods: \textbf{SemiEmb}, \textbf{DeepWalk}, and \textbf{Cross-GCN\_Fix}. \textbf{Cross-GCN\_Fix} is a variant of \textbf{Cross-GCN} with fixed value (1.0) of $\alpha_1$ and $\alpha_2$. Table~\ref{tab:real_perf} shows the performance of the compared methods. 
From the results, we have the following observations:
\begin{itemize}[leftmargin=*]
    \item With single layer, in most cases, all models with consideration of feature crosses, \ie \textbf{GIN1}, \textbf{Cross-GCN1\_Fix}, and \textbf{Cross-GCN1}, outperform \textbf{GCN1}, which signifies the effectiveness and rationality of considering feature crosses in a graph convolution operator. Note that single layer means that the model only considers the 1-hop neighbors when calculating the representation of a target node, rather than the total number of network layers of the model. As such, the \textbf{GIN1} model that employs an MLP to perform feature transformation can better capture non-linear feature transformation. On dataset where such non-linear feature important is important, \textbf{GIN1} achieves better performance on both \textbf{Cross-GCN1} and \textbf{Cross-GCN1\_Fix}.
    
    For the two layer versions, either \textbf{Cross-GCN} or \textbf{Cross-GCN\_Fix} achieves the best performance on different datasets, which further indicates the advantage of considering feature crosses. Moreover, the improvement over \textbf{GIN} validates the effectiveness of the proposed method to perform high-order feature crosses. Furthermore, \textbf{Cross-GCN\_Fix} achieves comparable performance as \textbf{Cross-GCN}, which suggests employing a simple order aggregation to avoid additional model parameters and the risk of over-fitting. Besides, across different datasets, the performance improvement over \textbf{GCN} on Cora is the smallest. Again, we postulate the reason to be that non-linear feature transformation is more important than feature crosses on Cora.
    \item In most cases, graph convolution-based models outperform conventional graph representation learning methods. In addition, these models with two layers achieves better performance than the corresponding single layer models. Similar observations have been presented in previous research~\cite{kipf2017semi,hamilton2017inductive}. Moreover, as compared to the origin GCN paper~\cite{kipf2017semi}, all the results have large variances. This should not be a concern since recent works~\cite{wu2019simplifying} have shown consistent results, and can be easily avoided by removing the runs with outlier performance~\cite{wu2019simplifying}. In addition, our reproduction of GCN (two layers) achieves inconsistent performance as compared to the origin GCN paper. Note that, while our GCN reproduction performs worse on Cora and Pubmed, it performs better on Citeseer. Moreover, previous work~\cite{wu2019simplifying} has also reported inconsistent reproducing results, which thus should also not be a concern.
\end{itemize}

\subsubsection{In-depth Analysis on GIN}
\label{sss:train_ratio}
\begin{figure}[t]
	\centering
	\includegraphics[width=0.33\textwidth]{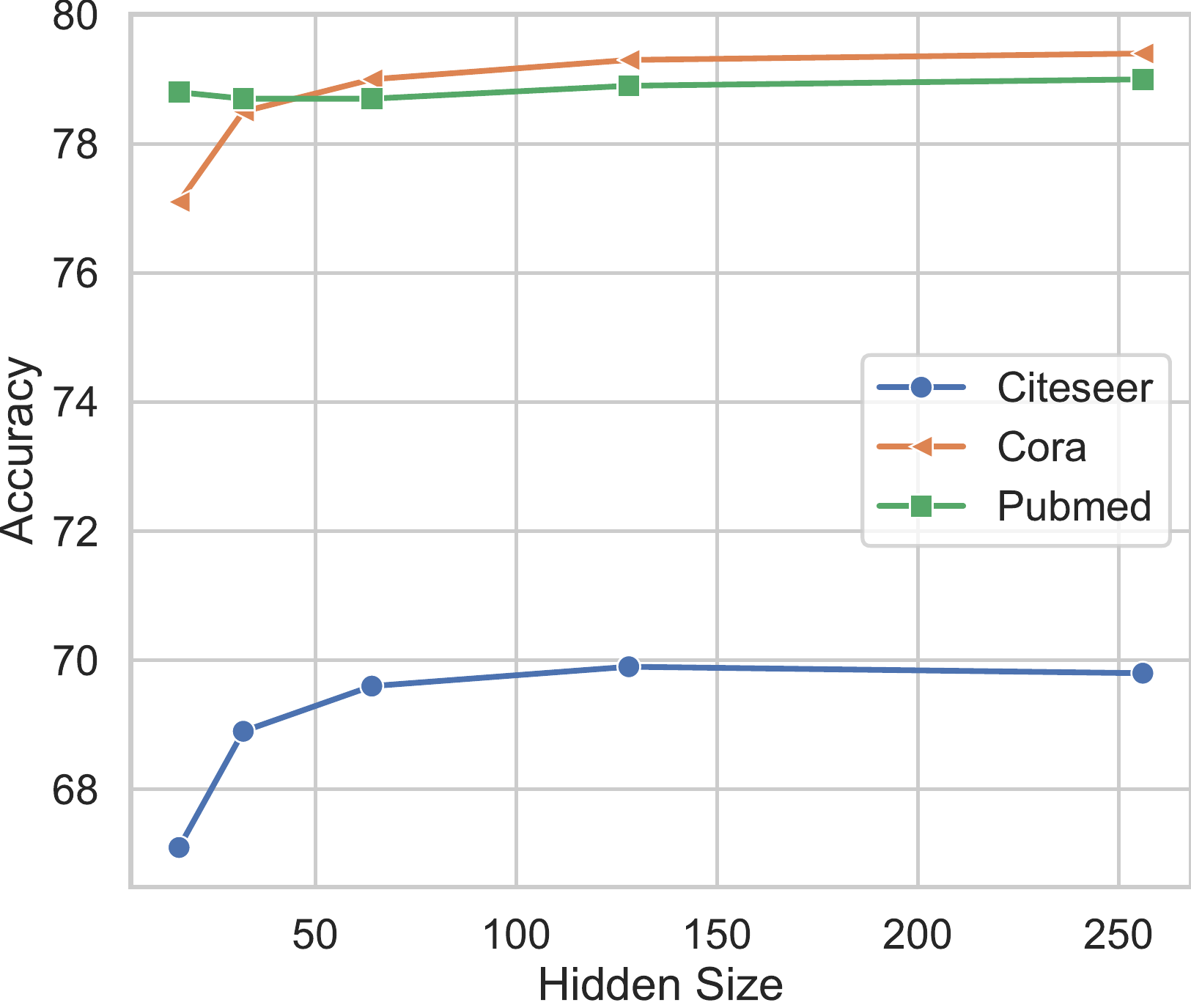}
	\vspace{-0.4cm}
	\caption{Impact on the performance of GIN with different hidden layer size.}
	\label{fig:gin_hidden}
\end{figure}

As aforementioned, according to the universal approximation theorem~\cite{hornik1991approximation}, an MLP can approximate any arbitrary transformation including the cross-feature transformation even though it may take a large number of hidden units \cite{andoni2014learning,beutel2018latent}. We study the impact of hidden layer size on the effectiveness of \textbf{GIN}. As shown in Figure~\ref{fig:gin_hidden}, \textbf{GIN} achieves better performance as we increase the hidden size from 32 to 256. This result indicates that feature crosses (\ie multiplication operation) indeed require more hidden units, which is consistent with previous work~\cite{andoni2014learning,beutel2018latent}. Moreover, the performance of \textbf{GIN} with hidden size of 256 is still not comparable as \textbf{Cross-GCN\_Fix} which further validates the rationality of considering feature crosses in an explicit manner.

\subsubsection{Impact of Hyper-parameters}
\begin{figure*}[]
	\centering
	\mbox{
		\hspace{-0.1in}
		\subfigure[Cora]{
			\label{fig:hp_dropout}
			\includegraphics[width=0.31\textwidth,trim={0 0 0 0cm}, clip]{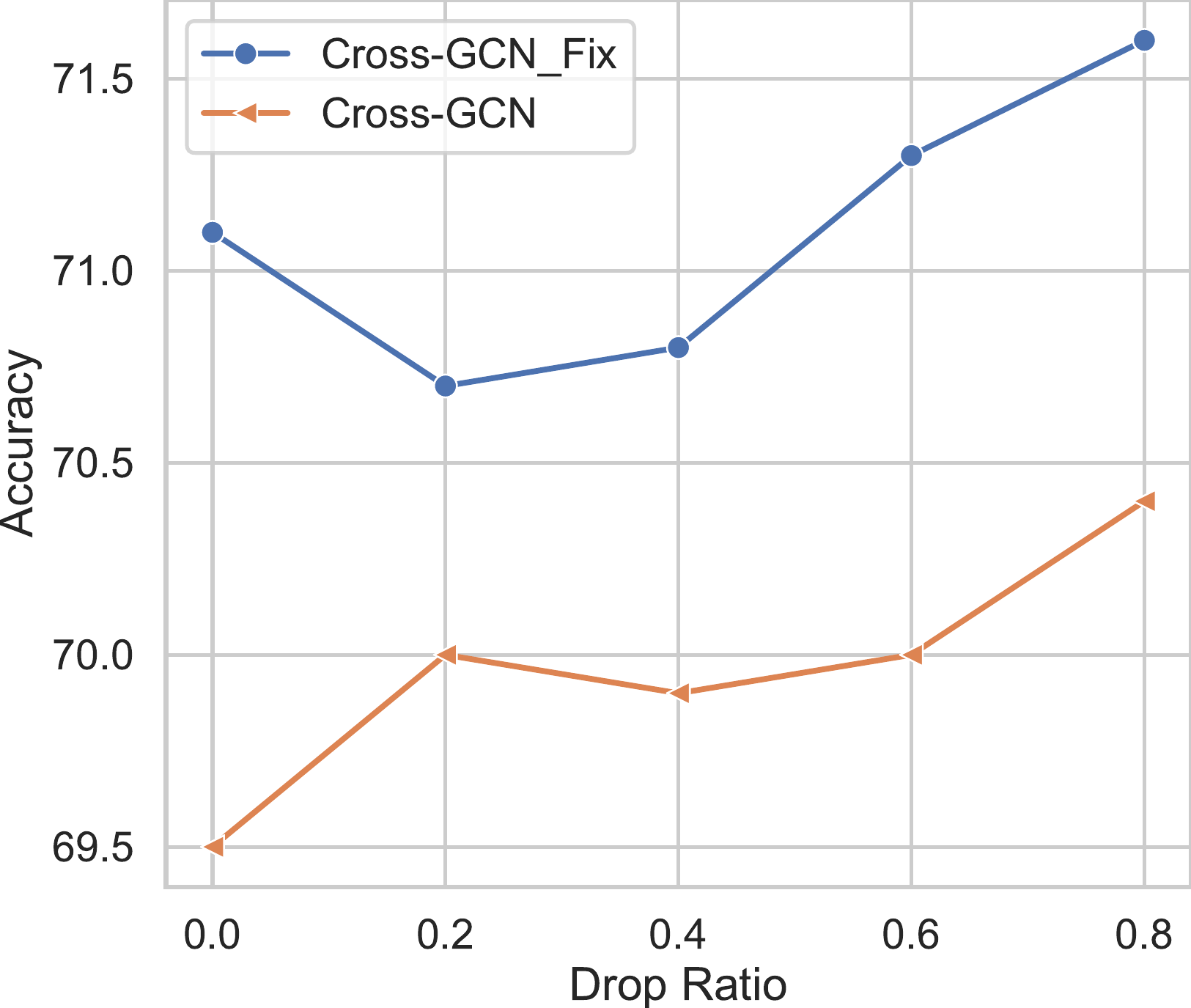}
		}
		\hspace{-0.2in}
		\subfigure[Citeseer]{
			\label{fig:hp_weight}
			\includegraphics[width=0.31\textwidth,trim={0 0 0 0cm}, clip]{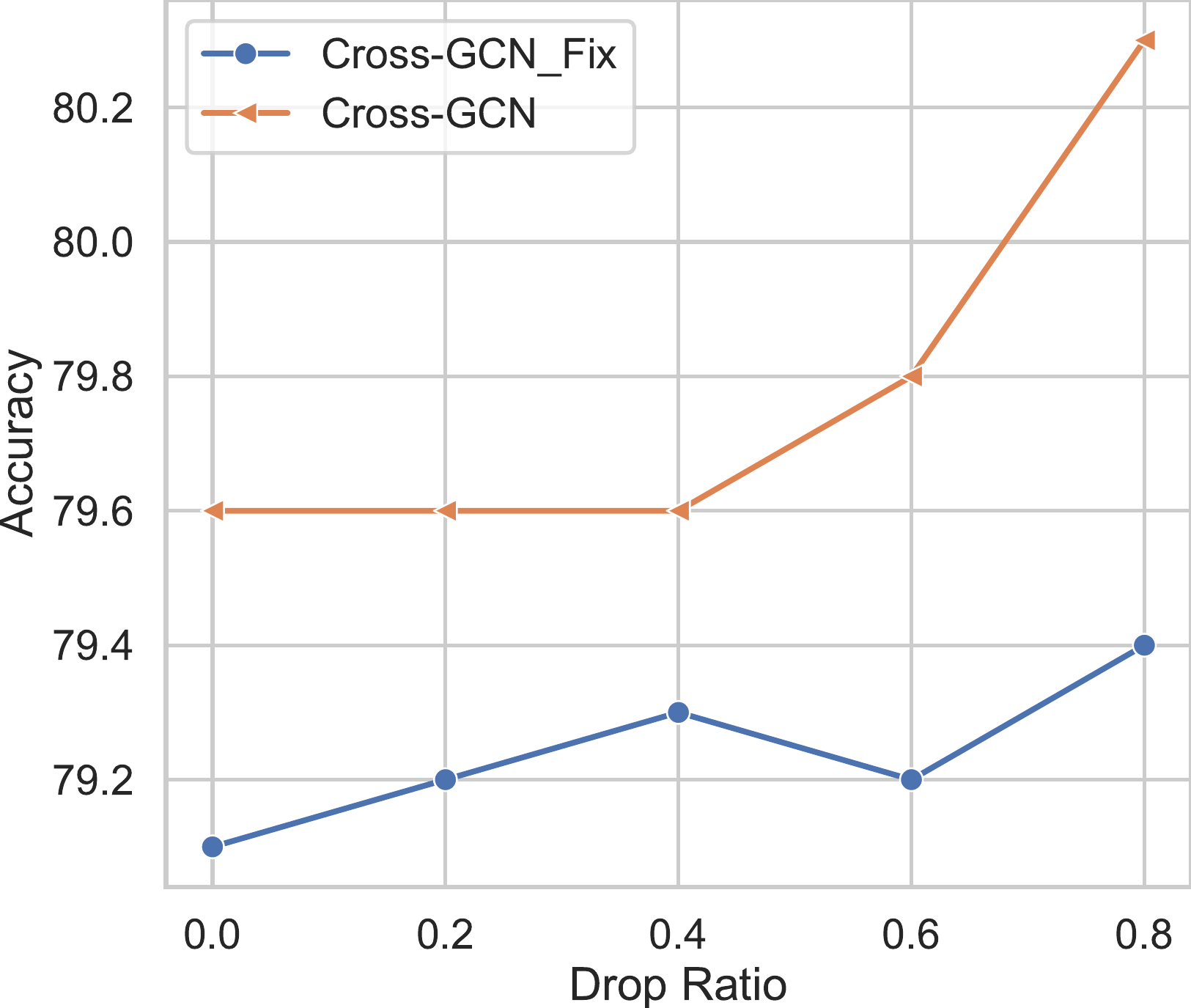}
		}
		\hspace{-0.2in}
		\subfigure[Pubmed]{
			\label{fig:hp_hidden}
			\includegraphics[width=0.31\textwidth,trim={0 0 0 0cm}, clip]{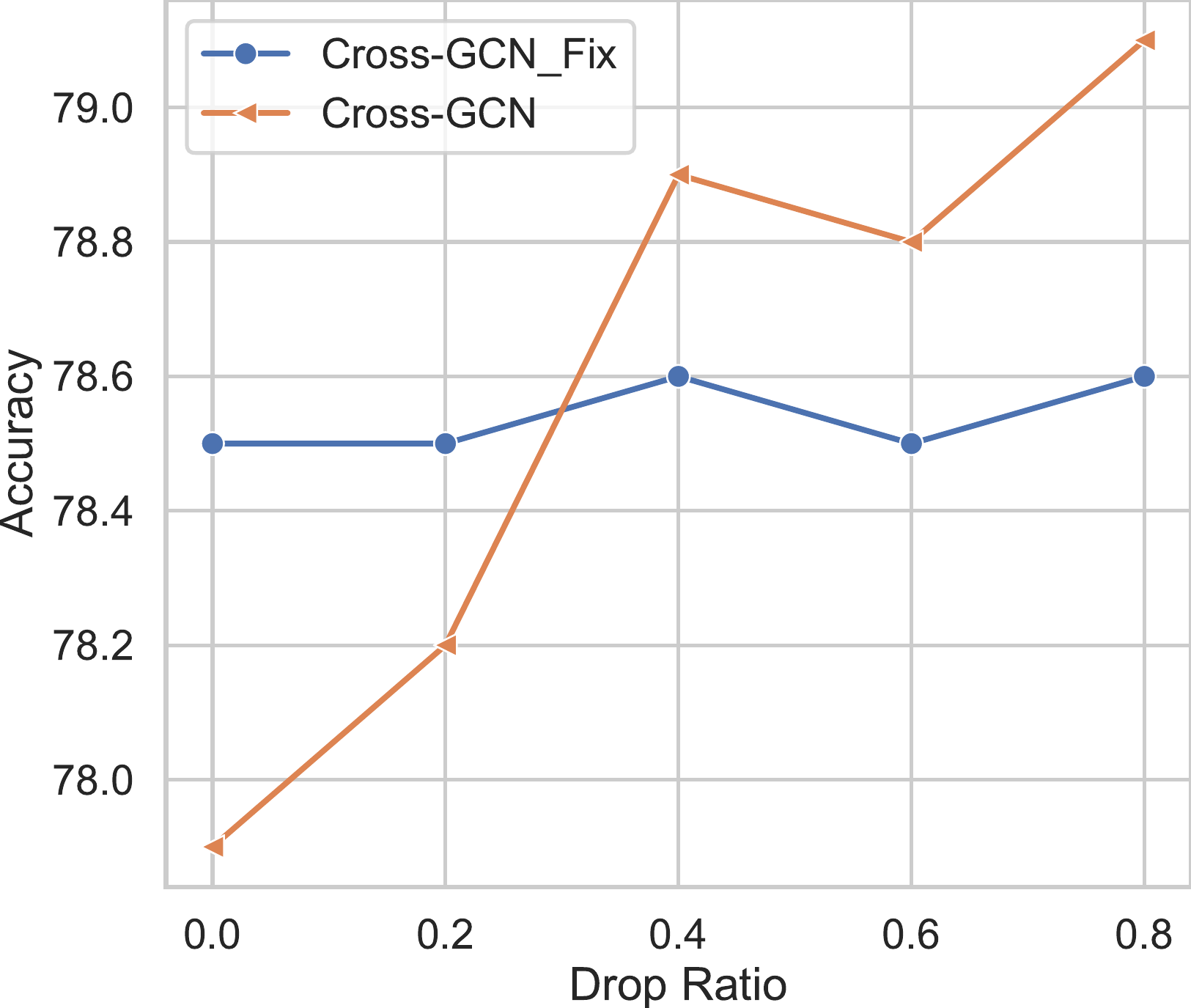}
		}
	}
	\caption{Performance of Cross-GCN on citation graphs as varying the value of its hyper-parameters.
	}
	\label{fig:hyperpara}
\end{figure*}
We then study the impact of hyper-parameters on the effectiveness of the proposed cross-feature graph convolution. We select the drop ratio ($\rho$) of dropout during model training as an example\footnote{This is because dropout is effective to prevent over-fitting which is a critical issue in the semi-supervised learning task with rare labeled nodes.}.  Figure~\ref{fig:hyperpara} illustrates the performance of \textbf{Cross-GCN} and \textbf{Cross-GCN\_Fix} as increasing the drop ratio from 0 to 0.8 on Citeseer (a), Cora (b), and Pubmed (c). From the results, we have the following observations, 1) in all the cases, both \textbf{Cross-GCN} and \textbf{Cross-GCN\_Fix} achieve better performance when increasing the drop rate, which reflects the importance of preventing over-fitting during the training of such GCN models. 2) As compared to \textbf{Cross-GCN}, on each dataset, the performance of \textbf{Cross-GCN\_Fix} varies in a smaller range. This result is reasonable since \textbf{Cross-GCN\_Fix} performs order aggregation without introducing any additional parameters, making the model less sensitive to the risk of over-fitting.

\subsubsection{Comparison of Efficiency}
Finally, we study the computational complexity by comparing the training time of \textbf{GCN}, \textbf{GIN}, and \textbf{Cross-GCN} on Reddit400K. Table~\ref{tab:train_time} shows the average training time over 50 continues epochs. Note that we omit the standard deviation since its value is tiny. Compared to \textbf{GIN} that implicitly encodes feature interactions, \textbf{Cross-GCN} achieves comparable running time. This result indicates that explicitly modeling cross feature is as efficient as the implicit manner. In addition, the average time of training \textbf{Cross-GCN} is roughly at most 1.5 times that of \textbf{GCN} which means the overhead is affordable. In addition, it is consistent with the analysis in Section~\ref{ss:complexity} and further validates the strong usability of \textbf{Cross-GCN}.
\begin{table}[]
	\caption{Average training time (seconds) over epochs on Reddit400K.}
	\label{tab:train_time}
	\resizebox{0.48\textwidth}{!}{%
		\begin{tabular}{c|cccccc}
			\hline
			\multirow{2}{*}{Method} & \multicolumn{6}{c}{Size of Hidden Layer} \\ \cline{2-7} 
			& 32 & 64 & 128 & 256 & 512 & 1024 \\ \hline
			GCN & 0.433 & 0.466 & 0.557 & 0.751 & 1.120 & 1.768  \\ 
			GIN & 0.478 & 0.534 & 0.627 & 0.881 & 1.430 & 2.629 \\ 
			Cross-GCN & 0.485 & 0.525 & 0.650 & 0.906 & 1.407 & 2.307 \\ \hline
		\end{tabular}%
	}
\end{table}
\section{Related Work}
\label{s:rel}
\subsection{Graph Neural Networks}
The recent year has witnessed the success of representation learning over graphs with neural networks owing to their extraordinary ability of non-linear modeling. The recent methods can be roughly divided into three main categories regarding their technique to encode graph structure: 1) \textit{skip-gram}, 2) \textit{autoencoder}, and 3) \textit{graph convolution}.

\textbf{Skip-gram}. Inspired by the skip-gram model~\cite{mikolov2013distributed}, which is proposed to learn word embeddings from large-scale documents, many recent methods learn the node embeddings based on a large scale of node sequences generated by random walk~\cite{perozzi2014deepwalk,grover2016node2vec,bojchevski2018deep,ying2018graph,bose2018adversarial}. Most of the existing methodologies in this line implement the idea by optimizing a classifier of which the target is to predict the co-occurred node (positive/target node) of a given node (context node), such as DeepWalk~\cite{perozzi2014deepwalk} and node2vec~\cite{grover2016node2vec}. Finding that negative sampling plays crucial role in training skip-gram model, a line of work has been focusing on exploring new sampling techniques. In particular, recent work~\cite{bojchevski2018deep} uses node-anchored sampling as an alternative of the uniform sampler in vertex2vec, which incorporates the connectedness to the context node during the sampling. More recently, dynamic negative sampling scheme is adopted, which adaptively selects hard negative samples (\ie similar to the context node) to boost the training process~\cite{ying2018graph,bose2018adversarial}. The main restriction of skip-gram-based methods is learning embedding independently from the predictive analysis, \ie in a two-phase fashion. As such, different applications on the same graph have to use the same node embedding while different applications might highlight different connection properties such as local similarity and structural similarity.

\textbf{Autoencoder}. Inspired by the success of Autoencoder (AE) in learning embedding from original features, various works use AE to learn node embedding. Sparse Autoencoder (SAE)~\cite{tian2014learning} is the first work in this line of research, which takes each column of the adjacency matrix as node features. Structure Deep Network Embedding SDNE~\cite{wang2016structural} enhances SAE by further incorporating Laplacian regularization. Furthermore, DNGR~\cite{cao2016deep} employs the positive pointwise mutual information between nodes as the features. As a natural extension of AE, recently, Variational Autoencoder (VAE) has also been introduced to learn node embeddings~\cite{kipf2016variational,zhu2018deep}, so that to learn more robust embeddings owning to its inherent generation model. Although AE-based methods could be end-to-end trained, their performance largely relies on designing graph-oriented node features which is typically labor intensive. 

\textbf{Graph Convolution}. Recently, research attention on graph representation learning has been shifted to convolution-based methods inspired by its extraordinary success in computer vision. Most of existing research focuses on developing node aggregation modules satisfying different requirements. For instance, GCN~\cite{kipf2017semi} and GraphSAGE~\cite{hamilton2017inductive} use pooling functions to aggregate directly connected neighbors. To enlarge the reception field, ChebNet~\cite{defferrard2016convolutional} and MixHop~\cite{abu2019mixhop} aggregate $k$-hop neighbors. Some other works focus on calculating weights for different nodes during the aggregating. For instance, DCNN~\cite{atwood2016diffusion} uses the transition probability of a length $k$ diffusion process. GAT~\cite{velickovic2018graph} and CAO~\cite{gao2019graph} adopt attention modules. Another line of research is to extend GCN from simple graphs to more complex graphs such as hypergraph~\cite{feng2019Hypergraph} and heterogeneous graph~\cite{wang2019heterogeneous,zhang2019heterogeneous}. 
Despite effectiveness, the existing designs of GCN forgo modeling the cross of features, which is the key difference to the proposed method.

\subsection{Cross-feature Modeling}
In the literature, two existing neural networks for non-graph data: \textit{Cross Network}~\cite{wang2017deep} and \textit{Compressed Interaction Netwrok} (CIN)~\cite{lian2018xdeepfm} also model cross features. Both of them calculate cross features with a stacking of multiple layers, and the $k$-order cross features are calculated at the $k$-th layer by crossing the output of layer $k-1$ and the input features. In other words, they have to take $K$ layers to incorporate cross features at up to $K$ orders. Therefore, once being equipped into a GCN layer to perform feature transformation, they would significantly increase the depth of a multi-layer GCN, leading to additional optimization issues. In addition, CIN computes the interactions between two continues layers with a complexity up to $O(D^2)$. CIN shows limited usability since the feature dimension ($D$) could be in the order of millions in some applications such as Web scale recommendation. Shallow models like Higher-order Factorization Machines~\cite{yang2015tensor,blondel2016higher} also considered $k$-order cross features. However, they have higher computation cost than CIN~\cite{lian2018xdeepfm}. In summary, the proposed cross-feature transformation module shows significant advantages over the potential solutions owing to its linear complexity to feature dimension.
\section{Conclusion}
\label{s:con}
In this work, we found that existing designs of GCN lack the ability to consider cross features, which are crucial for the success of several graph applications. To bridge the gap, we proposed a new graph convolution operator, named \textit{Cross-feature Graph Convolution}, which models feature crosses with complexity linear to feature dimension and order size. Furthermore, we developed a new solution, named \textit{Cross-GCN}, for graph-based learning tasks such as node classification. Experiments on three graphs demonstrated the effectiveness of Cross-GCN. In addition, the results indicates that applying Cross-GCN in applications with sparse low-level node features is promising.

In the future, we would test Cross-GCN with higher-order feature interactions and deeper structures. In addition, we plan to test Cross-GCN in other graph-based learning tasks such as link prediction. 
Moreover, we will further consider cross features in the node aggregation operation of GCNs. Furthermore,
we are interested in exploring the effectiveness of the cross-feature transformation module with different node aggregation methods such as Attention Networks~\cite{velickovic2018graph}. 
Lastly, we will explore methods to learn the weights ($\alpha$) in the order aggregation function and implementations of the function advanced than the employed weighted sum.
\bibliographystyle{IEEEtran}
\bibliography{00_refers}
\vspace{-34pt}
\begin{IEEEbiography}
[{\includegraphics[width=1.0in,height=1.4in,clip,keepaspectratio]{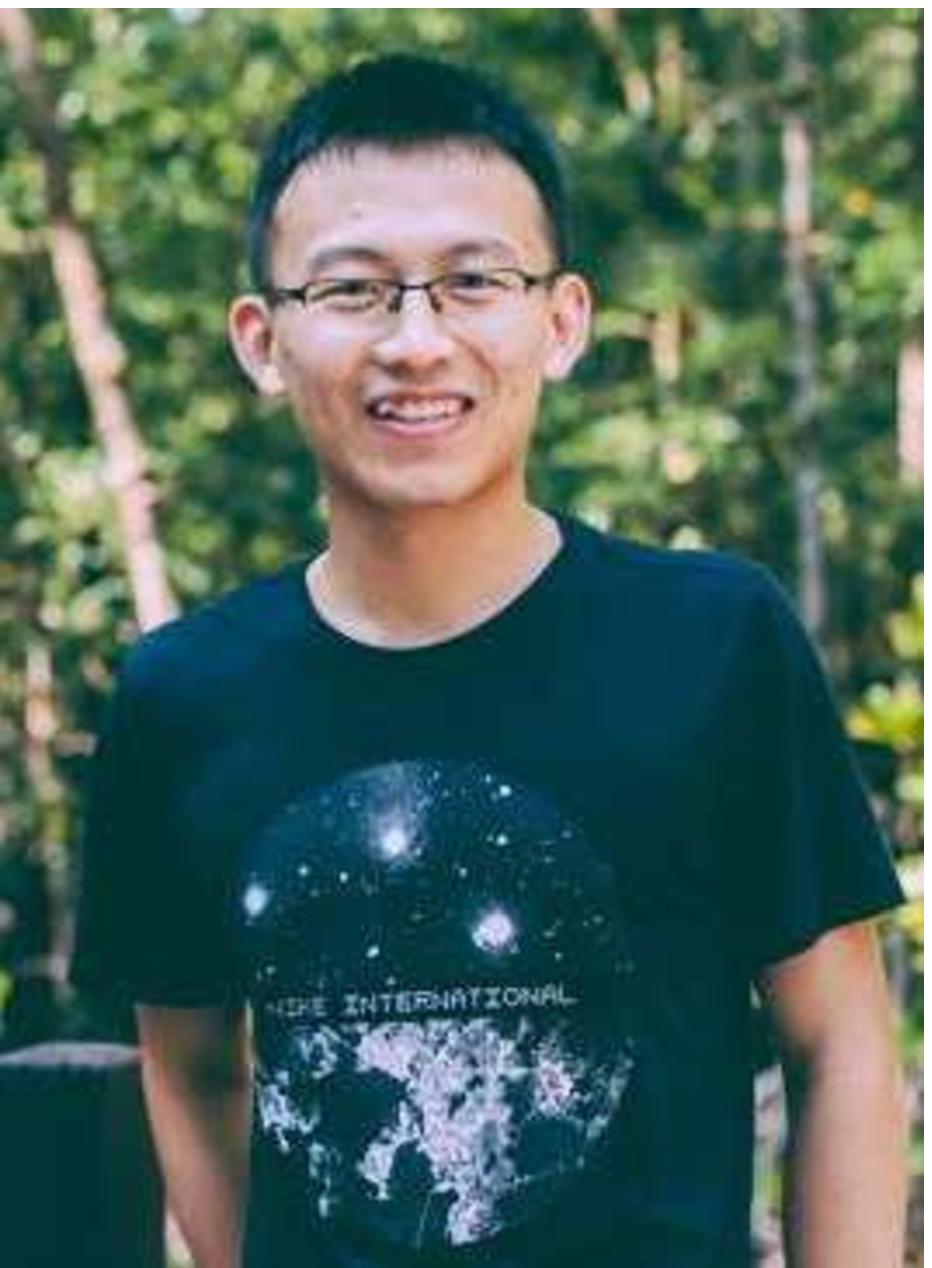}}]{Fuli Feng} is a Ph.D. student in the School of Computing, National University of Singapore. He received the B.E. degree in School of Computer Science and Engineering from Baihang University, Beijing, in 2015. His research interests include information retrieval, data mining, and multi-media processing. He has over 10 publications appeared in several top conferences such as SIGIR, WWW, and MM. His work on Bayesian Personalized Ranking has received the Best Poster Award of WWW 2018. Moreover, he has been served as the PC member and external reviewer for several top conferences including SIGIR, ACL, KDD, IJCAI, AAAI, WSDM etc.
\end{IEEEbiography}

\vspace{-35pt}
\begin{IEEEbiography}
[{\includegraphics[width=1.0in,height=1.4in,clip,keepaspectratio]{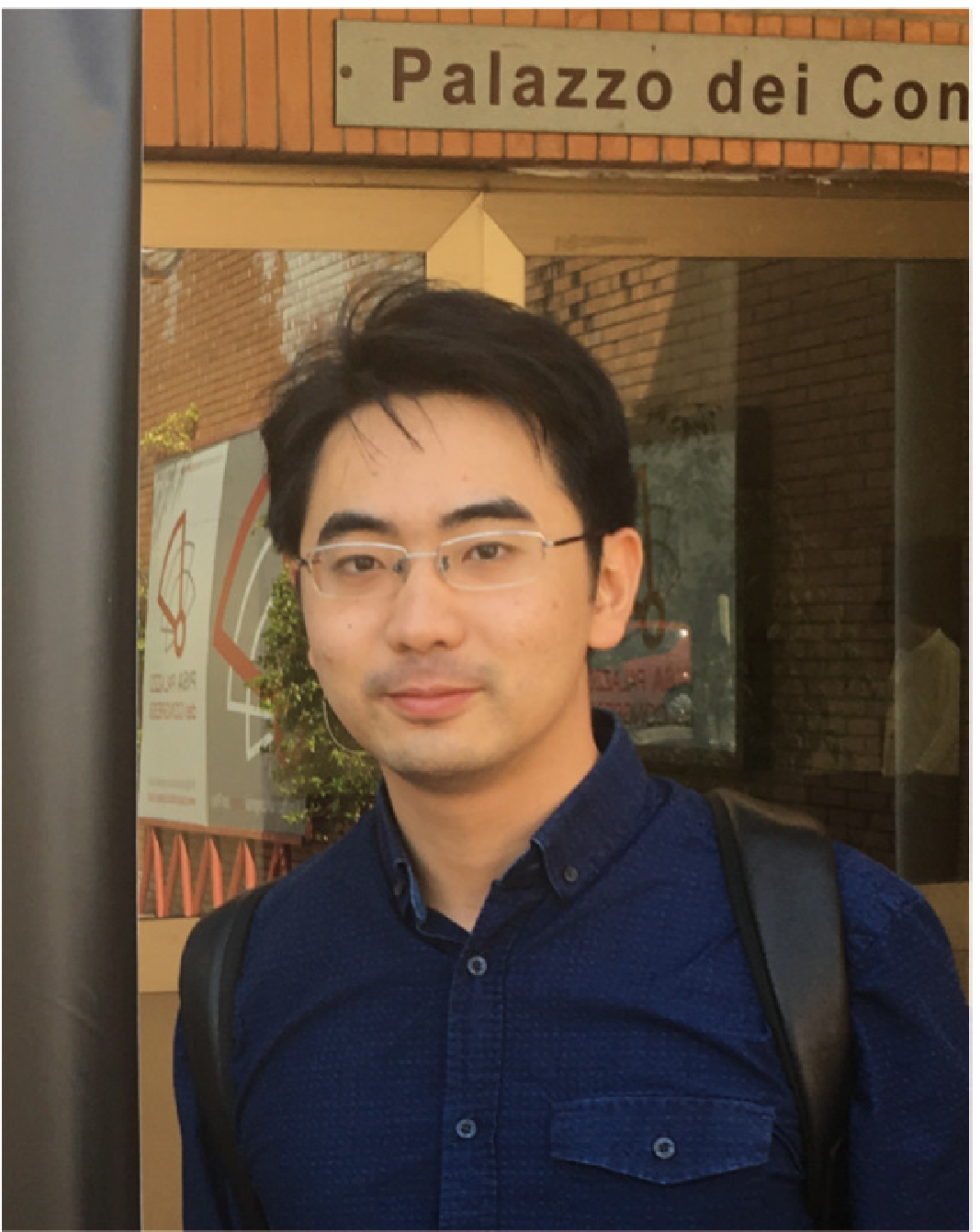}}]{Xiangnan He} is currently a research fellow with School of Computing, National University of Singapore (NUS). He received his Ph.D. in Computer Science from NUS. His research interests span recommender system, information retrieval, natural language processing and multi-media. His work on recommender system has received the Best Paper Award Honorable Mention in WWW 2018 and SIGIR 2016. Moreover, he has served as the PC member for top-tier conferences including SIGIR, WWW, MM, KDD, WSDM, CIKM, AAAI, and ACL, and the invited reviewer for prestigious journals including TKDE, TOIS, TKDD, TMM, and WWWJ. \end{IEEEbiography}

\vspace{-30pt}
\begin{IEEEbiography}
[{\includegraphics[width=1in,height=1.25in,clip,keepaspectratio]{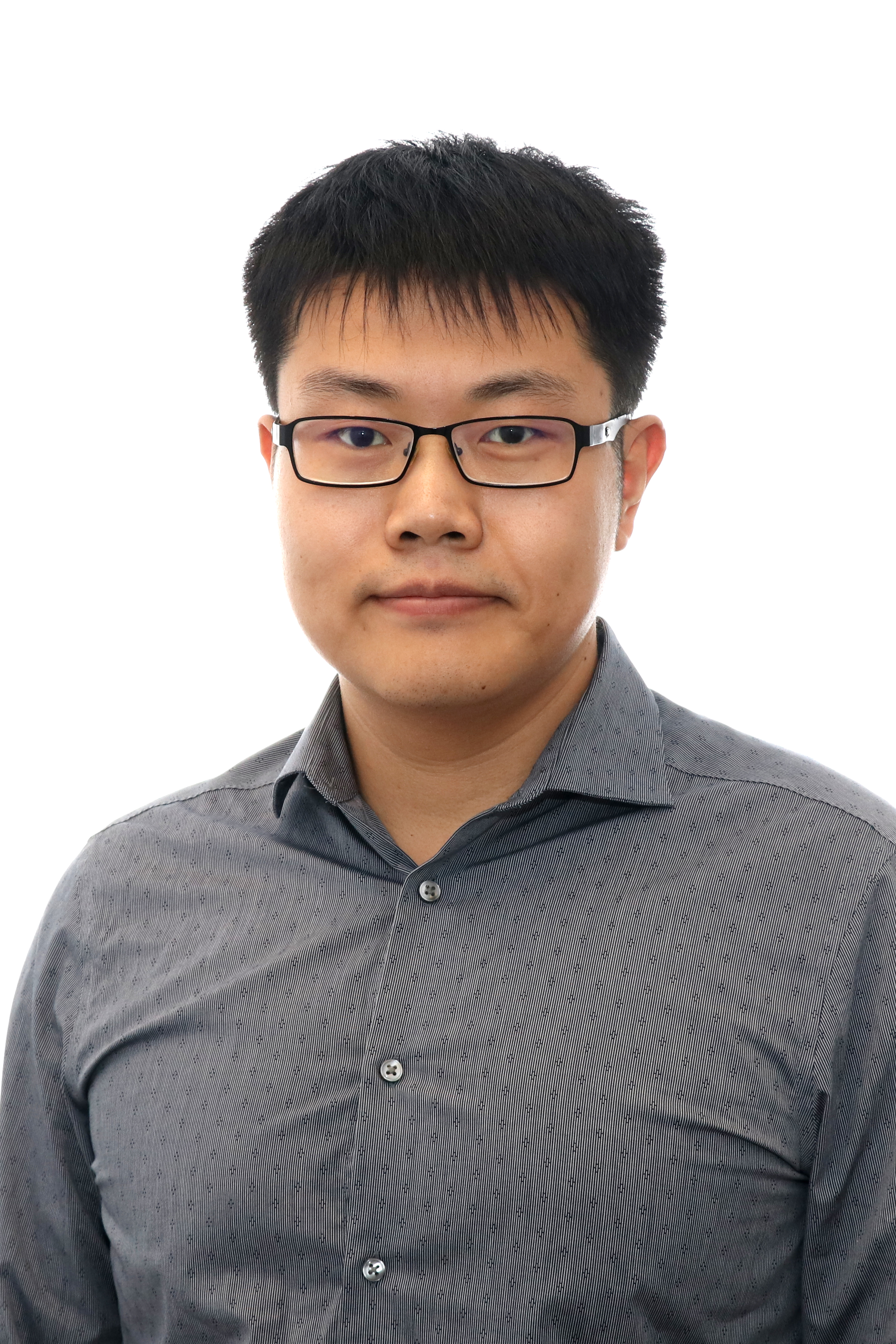}}]
{Hanwang Zhang} is currently an Assistant Professor at Nanyang Technological University, Singapore. He was a research scientist at the Department of Computer Science, Columbia University, USA. He has received the B.Eng (Hons.) degree in computer science from Zhejiang University, Hangzhou, China, in 2009, and the Ph.D. degree in computer science from the National University of Singapore in 2014. His research interest includes computer vision, multimedia, and social media. Dr. Zhang is the recipient of the Best Demo runner-up award in ACM MM 2012, the Best Student Paper award in ACM MM 2013, and the Best Paper Honorable Mention in ACM SIGIR 2016，and TOMM best paper award 2018. He is also the winner of Best Ph.D. Thesis Award of School of Computing, National University of Singapore, 2014.
\end{IEEEbiography}

\vspace{-30pt}
\begin{IEEEbiography}[{\includegraphics[width=1.0in,height=1.4in,clip,keepaspectratio]{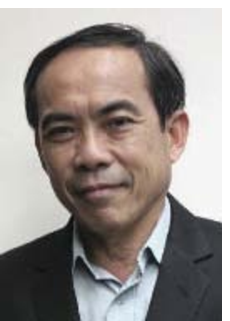}}]{Tat-Seng Chua}
Tat-Seng Chua is the KITHCT Chair Professor at the School of Computing, National University of Singapore. He was the Acting and Founding Dean of the School from 1998-2000. Dr Chuas
main research interest is in multimedia information
retrieval and social media analytics. In
particular, his research focuses on the extraction,
retrieval and question-answering (QA) of
text and rich media arising from the Web and
multiple social networks. He is the co-Director of
NExT, a joint Center between NUS and Tsinghua
University to develop technologies for live social media search. Dr Chua
is the 2015 winner of the prestigious ACM SIGMM award for Outstanding
Technical Contributions to Multimedia Computing, Communications and
Applications. He is the Chair of steering committee of ACM International
Conference on Multimedia Retrieval (ICMR) and Multimedia Modeling
(MMM) conference series. Dr Chua is also the General Co-Chair of
ACM Multimedia 2005, ACM CIVR (now ACM ICMR) 2005, ACM SIGIR
2008, and ACMWeb Science 2015. He serves in the editorial boards of
four international journals. Dr. Chua is the co-Founder of two technology
startup companies in Singapore. He holds a PhD from the University of
Leeds, UK.
\end{IEEEbiography}
\end{document}